\documentclass[9pt,technote]{IEEEtran}

\usepackage{}
\usepackage{amsmath}
\usepackage{amssymb}
\usepackage{amsfonts}
\usepackage{graphicx}
\usepackage{url}
\usepackage{subfigure}
\usepackage{float}
\usepackage{mathrsfs,fourier,citeref,pifont}
\usepackage{multicol,multienum}
\usepackage{threeparttable}
\usepackage{algorithmic}
\usepackage{algorithm}
\usepackage{color}
\usepackage{booktabs}

\newtheorem{theorem}{\indent Theorem}

\newtheorem{lemma}{Lemma}
\newtheorem{remark}{Remark}

\begin{document}

\title{Pulling back error to the hidden-node parameter technology: Single-hidden-layer feedforward network without output weight}

\author{Yimin~Yang,
~Q.~M.~Jonathan Wu,~Guangbin~Huang, and~Yaonan~Wang
\thanks{This work was supported by National Natural Science Foundation of China (61175075,61301254,61304007),
and and Hunan Provincial Innovation Foundation For Postgraduate (CX2012B147).
}
\thanks{Y.~M.~Yang is with the College of Electric Engineering, Guangxi University, Nanning 530004, China, and also with the Department of Electrical and Computer Engineering, University of Windsor N9B 3P4, Canada.  }
\thanks{Q.~M.~Jonathan Wu is with the Department of Electrical and Computer Engineering, University of Windsor N9B 3P4, Canada.  }

\thanks{G.~B.~Huang is with the School of Electrical and Electronic Engineering, Nanyang Technological University, Singapore 639798, Singapore. }

\thanks{Y.~N.~Wang are with the College of Electrical and Information Engineering, Hunan University, Changsha 410082, China.  }
}

\markboth{}%
{Shell \MakeLowercase{\textit{et al.}}: Bare Demo of IEEEtran.cls
for Journals}
\maketitle

\begin{abstract}
According to conventional neural network theories, the feature of single-hidden-layer feedforward neural networks(SLFNs) resorts to parameters of the weighted connections and hidden nodes. SLFNs are universal approximators when at least the parameters of the networks including hidden-node parameter and output weight are exist. Unlike above neural network theories, this paper indicates that in order to let SLFNs work as universal approximators, one may simply calculate the hidden node parameter only and the output weight is not needed at all. In other words, this proposed neural network architecture can be considered as a standard SLFNs with fixing output weight equal to an unit vector. Further more, this paper presents experiments which show that the proposed learning method tends to extremely reduce network output error to a very small number with only 1 hidden node. Simulation results demonstrate that the proposed method can provide several to thousands of times faster than other learning algorithm including BP, SVM/SVR and other ELM methods.

\end{abstract}


\IEEEpeerreviewmaketitle
\section{Introduction}

The widespread popularity of neural networks in many fields is mainly due to their ability to approximate complex nonlinear mappings directly from the input samples. In the past two decades, due to their universal approximation capability, feedforward neural networks (FNNs) have been extensively used in classification and regression problem\cite{1Huang+Chen+Siew2003}. According to Jaeger's estimation\cite{Jaeger2005}, 95\% literatures are mainly on FNNs. As a specific type of FNNs, the
single-hidden-layer feedforward network (SLFNs) plays an important
role in practical applications\cite{Huang2004}. For $N$ arbitrary distinct samples $(\textbf{x}_i,\textbf{t}_i)$, where $\textbf{x}_i=[x_{i1},x_{i2},\cdots,x_{in}]^T\in \textbf{R}^n$ and $\textbf{t}_i\in \textbf{R}^m$, an SLFNs with $L$ hidden nodes and activation function $h(x)$ are mathematically modeled as
\begin{equation}
f_L(\textbf{x})=\sum_{i=1}^L \beta_i h(\textbf{a}_i\cdot \textbf{x}_j+b_i),j=1,\cdots,N
\end{equation}
where $h(\textbf{a}_i,b_i,\textbf{x})$ denotes the output of the $i$th hidden node with the hidden-node parameters $(\textbf{a}_i,b_i)\in \textbf{R}^n\times \textbf{R}$ and $\beta_i\in\textbf{R}$ is the output weight between the $i$th hidden node and the output node. $\textbf{a}_i\cdot \textbf{x}$ denotes the inner product of vector $\textbf{a}_i$ and $\textbf{x}$ in $\textbf{R}^n$.

An active topic on the universal
approximation capability of SLFNs is then how to determine the parameters $\textbf{a}_i,b_i$, and $\beta_i(i=1,\cdots,L)$ such that the network output $f_L(\textbf{x})$ can approximate a given target $\textbf{T}, \textbf{T}=[\textbf{t}_1,\cdots,\textbf{t}_N]$. The feature of SLFNs resorts to parameters of the output weight and hidden nodes parameters. According to conventional neural network theories, SLFNs are universal approximators when all the parameters of the networks including the hidden-node parameters $(\textbf{a},b)$ and output weight $\beta$ are allowed adjustable\cite{1Huang2003}\cite{Zhang2012}.

Unlike above neural network theories that all the parameters in networks are allowed adjustable, other researches proposed some semi-random network theories\cite{Igelnik1995}\cite{PAO1994}\cite{Lowe1988}. For example, Lowe \cite{Lowe1988} focus on a specific RBF network: The centers $a$ in \cite{Lowe1988} can be randomly selected from the training data instead of tuning, but the impact factor $b$ of RBF hidden node is not randomly selected and usually determined by users.

Unlike above semi-random network theories, in 2006, Huang \emph{et al}\cite{1Huang2004} illustrated that iterative techniques are not required in adjusting all the parameters of SLFNs at all. Based on this idea, Huang \emph{et al} proposed simple and efficient learning steps referred to as extreme learning machine(ELM).
In \cite{1Mingbin:2005}\cite{1Huang2005}\cite{3Huang2003}, Huang \emph{et
al} have proved
that SLFNs with randomly
generated hidden node parameter can work as universal approximators by only calculating the output weights linking the hidden layer to the output nodes.  Recently ELM development \cite{Huang2012} shows that ELM unifies FNNs and SVM/LS-SVM. Compared to ELM, LS-SVM and PSVM achieve suboptimal solutions and have a higher computational cost.

Above neural network theories indicate that SLFNs can work as universal approximation at least hidden-node parameters\footnote{hidden-node parameters can be generated randomly} and output weight should be exist, however, in this paper we indicate that output weight do not need exist in SLFNs at all.

In \cite{Yang2013} we proposed a learning algorithm, called bidirectional extreme learning machine (B-ELM) in which half of hidden-node parameters are not randomly selected and are calculated by pulling back the network residual error to input weight. The experimental results in \cite{Yang2013} indicated that B-ELM tends to reduce network output error to a very small value at an extremely early learning stage. Further more, our recent experimental results indicate that in B-ELM\cite{Yang2013}, output weight play a very minion role in the network learning effectiveness. Inspired by these experimental results, in this paper, we show that SLFNs without output weight can approximate any target continuous function and classify any disjoint regions if one using pulling back error to hidden-node parameters. In particular, the following contributions have been made in this paper.

1) The learning speed of proposed learning method can be several to thousands of times faster than other learning methods including SVM, BP and other ELMs. Further more, it can provide good generalization performance and can be applied in regression and classification applications directly.

2) Different from conventional SLFNs in which the hidden node parameter and output weight should be needed, in the proposed method, we proved that SLFNs without output weight can still approximate any target continuous function and classify any disjoint regions. Thus the architecture of this single parameter neural network is extremely simpler than traditional SLFNs.

3)  Different from other neural networks requiring large number of hidden nodes\footnote{In \cite{Huang2012}, Huang \emph{et al} indicate "The generalization performance of ELM is not sensitive to the dimensionality $L$ of the feature space (the number of hidden nodes) as long as $L$ is set large enough (e.g., $L>1000$ for all the real-world cases tested in our simulations)."}, experimental study shows that the proposed learning method with only one hidden node can give significant improvements on accuracy instead of maintaining a large hidden-node-numbers hidden layer.

\section{Preliminaries and Notation}

\subsection{Notations and Definitions}
The sets of real, integer, positive real and positive integer numbers are denoted by $ \textbf{R},\textbf{Z}, \textbf{R}^+$ and $\textbf{Z}^+$, respectively. Similar to \cite{1Huang+Chen+Siew2003}, let $\digamma^2(X)$ be a space of functions $f$ on a compact subset $X$ in the $n$-dimensional Euclidean space $\textbf{R}^n$ such that $|f|^2$ are integrable, that is, $\int_{X}|f(x)|^2dx <\infty$. Let $\digamma^2(\textbf{R}^n)$ be denoted by $\digamma^2$. For $u,v\in \digamma^2(X)$, the inner product $<u,v>$ is defined by
\begin{equation}
<u,v>=\int_{X} u(\textbf{x})\overline{v(\textbf{x})}d\textbf{x}
\end{equation}
The norm in $\digamma^2(X)$ space will be denoted as $||\cdot||$.
$L$ denotes the number of hidden nodes. For $N$ training samples, $\textbf{x},\textbf{x}\in\textbf{R}^{N\times n}$ denotes the input matrix of network, $\textbf{T}\in \textbf{R}^{N\times m}$ denotes the desire output matrix of network. $\textbf{H} \in \textbf{R}^{N\times m}$ is called the hidden layer output matrix of the SLFNs; the $i$th column of $\textbf{H}$ ($\textbf{H}_i$) is the $i$th hidden node output with respect to inputs. The hidden layer output matrix $\textbf{H}_i$ is said to be randomly generated function sequence $\textbf{H}_i^r$ if the corresponding hidden-node parameters ($\textbf{a}_i,b_i$) are randomly generated. . $\textbf{e}_n, \textbf{e}_n\in\textbf{R}^{N\times m}$ denotes the residual error function for the current network $f_n$ with $n$ hidden nodes. $\textbf{I}$ is unit matrix and $\textbf{I} \in \textbf{R}^{m\times m}$.

\section{Bidirectional ELM for regression problem}

\begin{theorem}\cite{Yang2013}
Given $N$ training samples $\{(\textbf{x}_i,t_i)\}_{i=1}^N\in\textbf{R}^n\times \textbf{R}$ come from the same continuous function, given the sigmoid or sine activation function $h:\textbf{R} \to \textbf{R}$; Given a error feedback function sequence $\textbf{H}^e_{2n}(\textbf{x},\textbf{a},b)$ by
\begin{equation}
\textbf{H}^e_{2n}=e_{2n-1}\cdot (\beta_{2n-1})^{-1}
 \end{equation}
 If activation function $h$ is sin/cos, given a normalized function $u:\textbf{R} \to [0,1]$; If activation function $h$ is sigmoid, given a normalized function $u :\textbf{R} \to (0,1]$. Then for any continuous target function $f$, randomly generated function sequence ${\textbf{H}}^r_{2n+1}$, $\lim_{n\to \infty} \|f-(\textbf{H}_1^r \cdot\beta_1+\hat{\textbf{H}}^e_2(\hat{\textbf{a}}_2,\hat{b}_2) \cdot\beta_2+\cdots+\textbf{H}_{2n-1}^r \cdot\beta_{2n-1}+\hat{\textbf{H}}^e_{2n}(\hat{\textbf{a}}_{2n},\hat{b}_{2n})\cdot\beta_{2n}\|=0$ hold with probability one if
 \begin{equation}
 \begin{split}
 \hat{\textbf{a}}_{2n}= h^{-1}(u(\textbf{H}^e_{2n})) \cdot \textbf{x}^{-1} , \,\,\,\,\,\hat{\textbf{a}}_{2n}\in\textbf{R}^n\\
 \hat{b}_{2n}=\sqrt{mse( h^{-1}(u(\textbf{H}^e_{2n}))-\hat{\textbf{a}}_{2n} \cdot \textbf{x})},\,\,\,\,\,\hat{b}_{2n}\in\textbf{R}
 \end{split}
 \end{equation}
\begin{equation}
\hat{\textbf{H}}^e_{2n}=u^{-1}(h(\hat{\textbf{a}}_{2n} \cdot \textbf{x}+\hat{b}_{2n}))
\end{equation}
\begin{equation}
\beta_{2n}=\frac {e_{2n-1}\cdot\textbf{H}^e_{2n} } {\textbf{H}^e_{2n}\cdot (\textbf{H}^e_{2n})^T},\,\,\,\,\beta_{2n}\in\textbf{R}
\label{equ7}
\end{equation}
\begin{equation}
\beta_{2n+1}=\frac { e_{2n}\cdot\textbf{H}^r_{2n+1} } {\textbf{H}^r_{2n+1} \cdot (\textbf{H}^r_{2n+1})^T },\,\,\,\,\beta_{2n+1}\in\textbf{R}
\end{equation}
where $h^{-1}$ and $u^{-1}$ represent its reverse function, respectively. if $h$ is sine activation function, $h^{-1}(\cdot)=arcsin(\cdot)$; if $h$ is sigmoid activation function, $h^{-1}(\cdot)=-\log (\frac{1}{(\cdot)}-1)$.
 \end{theorem}
\begin{remark}
Compared with B-ELM, In the proposed method, we only make two changes. The first one is we set $\pmb{\beta}_{1}=\cdots=\pmb{\beta}_{n-1}=\cdots=\textbf{I}$. The second one is the pseudoinverse of input data $\textbf{x}^{-1}$ has been changed as $\textbf{x}^{-1}=\textbf{x}^T(\textbf{I}+\textbf{x}\textbf{x}^T)^{-1}$ based on the ridge regression theory.  Although very small changes are made, the experimental results show that by using this proposed learning method, one hidden-node SLFNs without output weight (output weight $\beta$ equal to unit matrix) can achieve similar generalization performance as other standard SLFNs with hundreds of hidden nodes. Further more, different from B-ELM \cite{Yang2013} which only work for regression problem, the proposed method can be applied in regression and multi-classification applications.
\end{remark}

\section{SLFNs without output weight}
Basic idea 1: our recent experimental results indicate that in B-ELM\cite{Yang2013}, output weight play a very minion role in the network learning effectiveness. Inspired by these experimental results, in this proposed method, we directly set output weight equal to unit matrix.

\begin{theorem}
Given $N$ training samples $\{(\textbf{x}_i,\textbf{t}_i)\}_{i=1}^N\in\textbf{R}^n\times \textbf{R}^m$ come from the same continuous function, given an SLFNs with any bounded nonconstant piecewise continuous function $\textbf{H}: \textbf{R}\rightarrow \textbf{R}$ for additive nodes or sine nodes, for any continues target function $f$, obtained error feedback function sequence ${\textbf{H}^\textbf{e}_{n}}, n\in Z$, $lim_{n\rightarrow \infty} \| f-(f_{n-1}+\textbf{H}^e_{n}\cdot \pmb{\beta}_{n-1}) \|=0$ holds with probability one if
\begin{equation}
\textbf{H}^e_{n}=\textbf{e}_{n-1}\cdot (\pmb{\beta}_{n-1})^{-1}
\label{errorfeedbackH}
\end{equation}
\begin{equation}
\pmb{\beta}_{n}=\pmb{\beta}_{n-1}=\textbf{I}, \pmb{\beta}_{n}\in \textbf{R}^{m\times m}
\label{equ7}
\end{equation}
\end{theorem}

\begin{proof}
The validity of this theorem is obvious because $\pmb{\beta}_{n-1}=\textbf{I}$ and $(\pmb{\beta}_{n-1})^{-1}=\textbf{I}$, $\textbf{H}^e_{n}$ equal to $\textbf{e}_{n-1}$. And we can get $\|\textbf{e}_n\|=0$.
\end{proof}

\begin{remark}
When $\textbf{H}^e_n=\textbf{e}_{n-1}=\textbf{T}$, it is easy to notice that the proposed method can reduce the network output error to 0. Thus the learning problem has been converted into finding optimal hidden node parameter $(\textbf{a},b)$ which lead to $\textbf{H}(\textbf{a},b,\textbf{x})\longrightarrow\textbf{T}$.
\end{remark}

Basic idea 2: For fixed output weight $\pmb{\beta}$ equal to unit matrix or vector ($\beta\in \textbf{R}^{m\times m}$), seen from equation (8)-(9), to train an SLFN is simply equivalent to finding a least-square solution $\textbf{a}^{-1}$ of the linear system $\textbf{H}(\textbf{a},\textbf{x})=\textbf{T}$. If activation function can be invertible, to train an SLFN is simply equivalent to pulling back residual error to input weight. For example, for $N$ arbitrary distinct samples $\{\textbf{x},\textbf{T}\}$, $\textbf{x}\in\textbf{R}^{N\times n}, \textbf{T}\in\textbf{R}^{N\times m},\textbf{T}\in[0,1]$, If activation function is sine function, to train an SLFN is simply equivalent to finding a least-square solution $\hat{\textbf{a}}$ of the linear system $\textbf{a}\cdot \textbf{x}=\arcsin(\textbf{T})$:
\begin{equation}
\|\textbf{H}(\hat{\textbf{a}}_1,\cdots,\hat{\textbf{a}}_n,\textbf{x})-\textbf{T}\|=\min_{\textbf{a}}\|\textbf{H}(\textbf{a}_1,\cdots,\textbf{a}_n,\textbf{x})-\textbf{T}\|
\end{equation}
According to [16], the smallest norm least-squares solution of the above linear system is $\hat{\textbf{a}_n}=\arcsin(\textbf{e}_{n-1})\cdot \textbf{x}^{-1}$. Based on this idea, we give the following theorem.

\begin{lemma}\cite{1Huang+Chen+Siew2003}
\label{lemma3}
Given a bounded nonconstant piecewise continuous function $H: \textbf{R}\rightarrow \textbf{R}$, we have
\begin{equation}
\lim_{(\textbf{a},b)\to(\textbf{a}_0,b_0)}\|\textbf{H}(\textbf{a}\cdot \textbf{x}+b)-\textbf{H}(\textbf{a}_0 \cdot \textbf{x}+b_0)\|=0
\end{equation}
\end{lemma}

\begin{theorem}
Given $N$ arbitrary distinct samples $\{\textbf{x},\textbf{T}\}, \textbf{x}\in\textbf{R}^{N\times n},\textbf{T}\in\textbf{R}^{N\times m}$, given the sigmoid or sine activation function $h$, for any continuous desire output $\textbf{T}$, there exist $\lim_{n\to \infty} \|\textbf{T}-(\hat{\textbf{H}}_1(\hat{\textbf{a}}_1,\hat{b}_1,\textbf{x})\pmb{ \beta}_1+\cdots+\hat{\textbf{H}}_{n}(\hat{\textbf{a}}_{n},\hat{b}_{n},\textbf{x}) \pmb{\beta}_{n}\|=0$ hold with probability one if
 \begin{equation}
 \begin{split}
 &\textbf{H}^e_n=\textbf{e}_{n-1}\\
 &\hat{\textbf{a}}_{n}= h^{-1}(u(\textbf{H}^e_{n})) \cdot \textbf{x}^{-1} \,\,\,,\hat{\textbf{a}}_{n}\in\textbf{R}^{n\times m}\\
 &\hat{b}_{n}=\sqrt{mse( h^{-1}(u(\textbf{H}^e_{n}))-\hat{\textbf{a}}_{n} \cdot \textbf{x})} \,\,\,,\hat{b}_{n}\in\textbf{R}^m\\
 &\pmb{\beta}_1=\pmb{\beta}_2=\cdots=\pmb{\beta}_n=\textbf{I}\\
 \end{split}
 \label{equ12}
 \end{equation}
\begin{equation}
\hat{\textbf{H}}_{n}=u^{-1}(h(\hat{\textbf{a}}_{n} \cdot \textbf{x}+\hat{b}_{n}))
\end{equation}
where if activation function $h$ is sin/cos, given a normalized function $u:\textbf{R} \to [0,1]$;  If activation function $h$ is sigmoid, given a normalized function $u:\textbf{R} \to (0 ,1]$. $h^{-1}$ and $u^{-1}$ represent its reverse function, respectively. If $h$ is sine activation function, $h^{-1}(\cdot)=\arcsin(u(\cdot)$; if $h$ is sigmoid activation function, $h^{-1}(\cdot)=-\log (\frac{1}{u(\cdot)}-1)$, $\textbf{x}^{-1}=\textbf{x}^T(\textbf{I}+\textbf{x}\textbf{x}^T)^{-1}$ .
 \end{theorem}

\begin{proof}
For an activation function $h(x): \textbf{R} \to \textbf{R}$, $\textbf{H}^e_{n}$ is given by
\begin{equation}
  \textbf{H}^e_{n}=h(\boldsymbol{\lambda}_{n})
  \end{equation}
  In order to let $\boldsymbol{\lambda}_{2n}\in \textbf{R}^m$,  here we give a normalized function $u(\cdot)$:
 $u(\textbf{H}) \in [0,1]$ if activation function is sin/cos; $u(\textbf{H}) \in (0,1)$ if activation function is sigmoid. Then for sine hidden node
\begin{equation}
\boldsymbol{\lambda}_{2n}=h^{-1}(u(\textbf{H}^e_{n}))=(arcsin(u(\textbf{H}^e_{n}))
\end{equation}
For sigmoid hidden node
 \begin{equation}
 \boldsymbol{\lambda}_{n}=h^{-1}(u(\textbf{H}^e_{n}))=-\log (\frac{1}{u(\textbf{H}^e_{n})}-1)
 \end{equation}
let $\boldsymbol{\lambda}_{n}=\textbf{a}_{n} \cdot \textbf{x}$, for sine activation function, we have
\begin{equation}
\hat{\textbf{a}}_{n}=h^{-1}(u(\textbf{H}^e_{n})) \cdot \textbf{x}^{-1}=arcsin(u(\textbf{H}^e_{n})) \cdot \textbf{x}^{-1}
\end{equation}
For sigmoid activation function, we have
\begin{equation}
\hat{\textbf{a}}_{n}=h^{-1}(u(\textbf{H}^e_{n})) \cdot \textbf{x}^{-1}=-\log (\frac{1}{u(\textbf{H}^e_{n})}-1) \cdot \textbf{x}^{-1}
\end{equation}
where $\textbf{x}^{-1}$ is the Moore-Penrose generalized inverse of the given set of training examples\cite{Hoerl2000}. Similar to \cite{4Huang2005}, we have 1: $\hat{\textbf{a}}_{n}=arcsin(u(\textbf{H}^e_{n})) \cdot \textbf{x}^{-1}$ is one of the least-squares solutions of a general linear system $\textbf{a}_{n} \cdot \textbf{x}=\boldsymbol{\lambda}_{n}$, meaning that the smallest error can be reached by this solution:
\begin{equation}
\|\hat{\textbf{a}}_{n} \cdot \textbf{x}-\boldsymbol{\lambda}_{n}\|=\|\hat{\textbf{a}}_{n}\textbf{x}^{-1}\textbf{x}-\boldsymbol{\lambda}_{n}\|=\min_{\textbf{a}_{n}} \|\textbf{a}_{n}\cdot \textbf{x}-arcsin(u(\textbf{H}^e_{n}))\|
\end{equation}
2: the special solution $\hat{\textbf{a}}_{n}=h^{-1}(u(\textbf{H}^e_{n})) \cdot \textbf{x}^{-1}$ has the smallest norm among all the least-squares solutions of $\textbf{a}_{n}\cdot \textbf{x}=\boldsymbol{\lambda}_{n}$, which is guarantee that $\textbf{a}_{n} \in [-1,1]$.
Although the smallest error can be reached by equation (17)-(18), we still can reduce its error by adding bias $b_{n}$. For sine activation function:
\begin{equation}
\begin{split}
\hat{b}_{n}&=\sqrt{mse(h^{-1}(u(\textbf{H}^e_{n}))-\hat{\textbf{a}_{n}} \cdot \textbf{x})}\\
&=\sqrt{mse(arcsin(u(\textbf{H}^e_{n}))-\hat{\textbf{a}_{n}} \cdot \textbf{x})}
\end{split}
\end{equation}
For sigmoid activation function
\begin{equation}
\begin{split}
\hat{b}_{n}&=\sqrt{mse(h^{-1}(u(\textbf{H}^e_{n}))-\hat{\textbf{a}_{n}} \cdot \textbf{x})}\\
&=\sqrt{mse((-\log ({1/u(\textbf{H}^e_{n})}-1))-\hat{\textbf{a}}_{n} \cdot \textbf{x})}
\end{split}
\end{equation}

According to 19 and \emph{Lemma} \ref{lemma3}, we have
\begin{equation}
\begin{split}
&\min_{\textbf{a}_{n}}\|u^{-1}(h(\textbf{a}_{n}\cdot \textbf{x}))-u^{-1}(h(\boldsymbol{\lambda}_{n}))\|\\
&=\|u^{-1}(h(\hat{\textbf{a}}_{n} \cdot \textbf{x}))-u^{-1}(h(\boldsymbol{\lambda}_{n}))\|\\
&> \|u^{-1}(h(\hat{\textbf{a}}_{n} \cdot \textbf{x}+\hat{b}_{n}))-u^{-1}(h(\boldsymbol{\lambda}_{n}))\|=\|\boldsymbol{\sigma}\| \\
\end{split}
\end{equation}

We consider the residual error as
\begin{equation}
\begin{split}
\triangle=&\|e_{n-1}\|^2-\| e_{n-1}-\textbf{H}^e_{n} \cdot \beta_{n} \|^2 \\
=&2\beta_{n}\langle e_{n-1},\textbf{H}^e_{n}\rangle - \|\textbf{H}^e_{n}\|^2\cdot \beta_{n}^2\\
=& \| \textbf{H}^e_{n} \|^2(\frac {2\beta_{n}\langle e_{n-1}, \textbf{H}^e_{n}\rangle }{\| \textbf{H}^e_{n} \|^2}-\beta_{n}^2)\\
\end{split}
\label{equ7}
\end{equation}
Let
\begin{equation}
\begin{split}
\hat{\textbf{H}}_{n}^e&=u^{-1}(h(\hat{\textbf{a}}_{n} \cdot \textbf{x}+\hat{b}_{n}))\\
&=\frac{\textbf{e}_{n-1}-\boldsymbol{\sigma}}{\pmb{\beta}_n}\\
&=\frac{\hat{\textbf{e}}_{n-1}}{\pmb{\beta}_{n-1}}
\end{split}
\end{equation}
Because $\pmb{\beta}_{n}=\pmb{\beta}_{n-1}=\textbf{I}$, we have equation (23) $\geq 0$ is still valid for
\begin{equation}
\begin{split}
&\triangle=\|\hat{\textbf{H}}^e_n\|^2(\frac{2\|\pmb{\beta}_n\|\langle \hat{\textbf{e}}_{n-1},\frac{\hat{\textbf{e}}_{n-1}}{\pmb{\beta}_{n-1}}\rangle}{\|\frac{\hat{\textbf{e}}_{n-1}}{\pmb{\beta}_{n-1}}\|^2}-\|\pmb{\beta}_n\|^2)\\
&=\|\hat{\textbf{H}}^e_n\|^2(\frac{2\|\hat{\textbf{e}}_{n-1}\|^2}{\frac{\|\hat{\textbf{e}}_{n-1}\|^2}{\|\pmb{\beta}_{n-1}\|^2}}-\pmb{\beta}_n^2)\\
&=\|\hat{\textbf{H}}^e_n\|^2\pmb{\beta}_{n}^2\geq 0\\
\end{split}
\end{equation}

Now based on equation 25, we have $\|\textbf{e}_{n-1}\|\geq \| \textbf{e}_{n}\|$, so the sequence $\|\textbf{e}_n\|$ is decreasing and bounded below by zero and the sequence  $\|\textbf{e}_n\|$ converges.

\end{proof}

\begin{remark}

According to Theorem 2-3, for $N$ arbitrary distinct samples $(\textbf{x}_i,\textbf{t}_i)$ where $\textbf{x}_i=[x_{i1},x_{i2},\cdots,x_{iN}]^T\in \textbf{R}^n$ and $\textbf{t}_i\in\textbf{R}^m$, the proposed network with $L$ hidden nodes and activation function $h(x)$ are mathematically modeled as
\begin{equation}
f_L(\textbf{x})=\sum_{i=1}^L u^{-1}(h(\textbf{a}_i\cdot \textbf{x}_j+b_i)),j=1,\cdots,N
\end{equation}
where $u$ is a normalized function, $\textbf{a}_i\in\textbf{R}^{n\times m}, b_i\in\textbf{R}^m$. Here, the proposed the proposed method for SLFN can be summarized in Algorithm 1.
\end{remark}
\begin{algorithm}[!htb]
\caption{the proposed method algorithm}
  \begin{algorithmic}
\STATE \textbf{Initialization}: Given a training set $\{(\textbf{x}_i,\textbf{t}_i)\}^{N}_{i=1}\subset \textbf{R}^n\times \textbf{R}^m$,the hidden-node output function $\textbf{H}(\textbf{a},b,\textbf{x})$, continuous target function $f$, set number of hidden nodes $L=1$, $\textbf{e}=\textbf{T}$.
\STATE  \textbf{Learning step}:
\WHILE {$L<L_{max}$}
\STATE Increase by one the number of hidden nodes $L;L=L+1$;
 \STATE Step 1) set $\textbf{H}^e_L=\textbf{e}$;
 \STATE Step 2) calculate the input weight $\textbf{a}_L$, bias $b_L$ based on equation \ref{equ12};
 \STATE Step 3) calculate $\textbf{e}$ after adding the new hidden node $L$:\\$\textbf{e}=\textbf{e}-u^{-1}(h(\textbf{a}\cdot \textbf{x}+b))$
  \ENDWHILE
   \end{algorithmic}
\end{algorithm}
\begin{remark}
Different from other neural network learning methods in which output weight parameter should be adjusted, in the proposed method, the output weight of SLFNs can be equal to unit matrix and thus the proposed neural network does not need output weight at all. Thus the architecture and computational cost of this proposed method are much smaller than other traditional SLFNs.
\end{remark}

\begin{remark}

Subsection V.C presents experiments which show that the proposed method with only one hidden node can give better generalization performance than the proposed network with $L(L>1)$ hidden node. Based on this experimental results, for $N$ arbitrary distinct samples $(\textbf{x}_i,\textbf{t}_i)$ where $\textbf{x}_i=[x_{i1},x_{i2},\cdots,x_{iN}]^T\in \textbf{R}^n$ and $\textbf{t}_i\in\textbf{R}^m$, the proposed network is mathematically modeled as
\begin{equation}
f_L(\textbf{x})=u^{-1}(h(\textbf{a}_1\cdot \textbf{x}_j+b_1)),j=1,\cdots,N
\end{equation}
where $u$ is a normalized function, $\textbf{a}_1\in\textbf{R}^{n\times m}, b_1\in\textbf{R}^m$. Thus algorithm 1 can be modified as algorithm 2.
\end{remark}
\begin{algorithm}[!htb]
\caption{the proposed method algorithm}
  \begin{algorithmic}
\STATE \textbf{Initialization}: Given a training set  $\{(\textbf{x}_i,\textbf{t}_i)\}^{N}_{i=1}\subset \textbf{R}^n\times \textbf{R}^m$,the hidden-node output function $\textbf{H}(\textbf{a},b,\textbf{x})$, continues target function $f$, set number of hidden nodes $L=1$.
\STATE  \textbf{Learning step}:
 \STATE Step 1) set $\textbf{H}^e_1=\textbf{T}$;
 \STATE Step 2) calculate the input weight $\textbf{a}_1$, bias $b_1$ based on equation \ref{equ12};
   \end{algorithmic}
\end{algorithm}

\section{Experimental Verification}

\begin{table}[!htb]
\centering
\caption{Specification of regression problems}
\begin{tabular}{ccccc}
\toprule
Datasets    & \#Attri    & \#Train   &\# Test \\
\midrule
Auto MPG &8   &200  &192   \\
Machine CPU &6   &100  &109    \\
Fried &11   &20768  &20000    \\
Wine Quality &12 &2898 &2000 \\
Puma &9 &4500 &3692 \\
California Housing    &8    &16000   &4000    \\
House 8L &9    &16000  &6784    \\
Parkinsons motor   &26   &4000   &1875    \\
Parkinsons total   &26   &4000   &1875    \\
Puma   &9   &6000   &2192   \\
Delta elevators   &6   &6000   &3000    \\
Abalone   &9   &3000   &1477   \\
\bottomrule
\end{tabular}
\end{table}

\begin{table}[!htb]
\centering
\caption{specification of Small/Medium-sized classification problems}
\begin{tabular}{ccccc}
\toprule
Datasets    & \#Feature    & \#Train    &\# Test\\
\midrule
A9a  &123  &32561 &16281   \\
colon-cancer  &2000  &40 &22   \\
USPS  &256  &7291 &2007    \\
Sonar  &60  &150 &58   \\
Hill Valley  &101  &606 &606    \\
Protein  &357  &17766 &6621   \\
\bottomrule
\end{tabular}
\end{table}

\begin{table}[!htb]
\centering
\caption{specification of large-sized classification problems}
\begin{tabular}{ccccc}
\toprule
Datasets    & \#Feature   & \#Train    &\# Test  \\
\midrule
Covtype.binary  &54  &300000 &280000    \\
Mushrooms  &112  &4000 &4122   \\
Gisette  &5000  &6000 &1000   \\
Leukemia  &7129  &38 &34   \\
Duke  &7129 &29 &15    \\
Connect-4 &126  &50000 &17557   \\
Mnist &780  &40000 &30000   \\
DNA &180  &1046 &1186    \\
w3a &300  &4912 &44837   \\
\bottomrule
\end{tabular}
\end{table}

To examine the performance of our proposed algorithm (B-ELM), in
this section, we test them on some benchmark regression and classification problems. Neural networks
are tested in SVR, SVM, BP, EM-ELM,I-ELM, EI-ELM, B-ELM, ELM and proposed the proposed method.

\subsection{Benchmark Data Sets}

In order to extensively verify the performance of different algorithms, wide type of data sets have been tested in our simulations, which are of \emph{small size, medium dimensions, large size, and/or high dimensions}. These data sets include 12 regression problems and 15 classification problems. Most of the data sets are taken from UCI Machine Learning Repository\footnote{http://archive.ics.uci.edu/ml/datasets.html} and LIBSVM DATA SETS \footnote{http://www.csie.ntu.edu.tw/~cjlin/libsvmtools/datasets/}.

\emph{Regression Data Sets}: The 12 regression data sets(cf.Table I)can be classified into two groups of data:

1) data sets with relatively small size and low dimensions, e.g., Auto MPG, Machine CPU, Puma, Wine, Abalone;

2) data sets with relatively medium size and low dimensions, e.g., Delta, Fried, California Housing, Parkinsons;

\emph{Classification Data Sets}: The 15 classification data sets(cf.Table II and Table III) can be classified into three groups of data:

1) data sets with relative medium size and medium dimensions, e.g., Sonar, Hill Valley, Wa3, DNA, Mushrooms, A9a, USPS;

2) data sets with relative small size and high dimensions, e.g., Colon-cancer, Leukemia, Duke;

3) data sets with relative large size and high dimensions, e.g., Protein, Covtype.binary, Gisette, Mnist, Connect-4;

In these data sets, the input data are
normalized into $[-1,1]$ while the output data for regression are normalized into the range $[0,1]$.
All data sets have been preprocessed in the same way (held-out method). Ten different random permutations of the whole data set are taken without replacement, and some(see in tables) are used to create the training set and the remaining is used for the test set.  The average results are obtained over 50 trials for all problems.

\subsection{Simulation Environment Settings}

The simulations of different algorithms on the data sets which are shown in Table I and Table II are carried out in Matlab 2009a environment running on the same Windows 7 machine with at 2 GB of memory and an i5-430 (2.33G) processor. The codes used for SVM and SVR are downloaded from LIBSVM\footnote{http://www.csie.ntu.edu.tw/~cjlin/libsvmtools/datasets/}, The codes used for B-ELM, ELM and I-ELM are downloaded from ELM\footnote{http://www.ntu.edu.sg/home/egbhuang/elm$\_$codes.html}.

For SVM and SVR, in order to achieve good generalization performance, the cost parameter $C$ and kernel parameter $\gamma$ of SVM and SVR need to be chosen appropriately. We have tried a wide range of $C$ and $\gamma$. For each data set, similar to \cite{Hsu2002}, we have used 30 different value of $C$ and $\gamma$, resulting in a total of 900 pairs of $(C,\gamma)$. The 30 different value of $C$ and $\gamma$ are $\{2^{-15},2^{-14},\cdots, 2^{14},2^{15}\}$. Average results of 50 trials of simulations with each combination of $(C,\gamma)$ are obtained and the best performance obtained by SVM/SVR are shown in this paper.

For BP, the number of hidden nodes are gradually increased by an interval of 5 and the nearly optimal number of nodes for BP are then selected based on cross-validation method. Average results of 50 trails of simulations for each fixed size of SLFN are obtained and finally the best performance obtained by BP are shown in this paper as well.

Simulations on large data sets(cf.Table III) are carried out in a high-performance computer with Intel Xeon E3-1230 v2 processor (3.2G) and 16-GB memory.

\subsection{Generalization performance comparison of ELM methods with different hidden nodes}
The aim of this subsection is to show that the proposed method with only one hidden node generally achieves better generalization performance than other learning methods. And it is also to show that the proposed method with one hidden node achieves the best performance than the proposed method with $L,L>1$ one hidden node. In this subsection, I-ELM, ELM, EI-ELM and the proposed method are compared in one regression problem and three classification problems: Fried, DNA, USPS and Mushroom. In these cases, all the algorithms increase the hidden nodes one by one.  More importantly, we find that the testing accuracy obtained by proposed method is reduced to a very high value when only one hidden node is used. And the testing accuracy obtained by proposed method is not increased but is reduced when hidden node added one by one. This means the proposed method only need to calculates one-hidden-node parameter($\textbf{a}_1,b_1$) once and then SLFNs without output weight can achieve similar generalization performance as other learning method with hundreds of hidden nodes. Thus in the following experiments, the number of hidden node equal to one in the proposed method.

\begin{table*}[!htb]
 \centering
\caption{Performance comparison (mean-mean testing RMSE; time-training time)}
\begin{tabular}{ccccccccc}
        \toprule
       Datasets & \multicolumn{2}{c}{I-ELM} (200 nodes) & \multicolumn{2}{c}{B-ELM} (200 nodes)  &\multicolumn{2}{c}{EI-ELM} (200 nodes, $p=50$) & \multicolumn{2}{c}{the proposed method} (1 nodes)\\
                   \cmidrule{2-3}  \cmidrule{4-5}  \cmidrule{6-9}
      & Mean & time(s) & Mean & time(s) & Mean & time(s)&  Mean & time(s)\\
        \midrule
        House 8L &   0.0946      &      1.1872      & \underline{0.0818}   &3.8821 & 0.0850   &10.7691   &    \underline{0.0819}       &   \textbf{0.0020}     \\
        Auto MPG  &  0.1000 & 0.2025  &  \textbf{0.0920}     &  0.3732 & 0.0918  & 1.3004   &  0.0996
        &  <\textbf{0.0001}   \\
        Machine CPU   &  0.0591       &  0.1909      & 0.0554 &  0.3469   &0.0551 &1.2633     &   \textbf{0.0489}     &  <\textbf{0.0001}   \\
     Fried  &    0.1135    &   0.8327     & 0.0857   &5.5063  & 0.0856 & 7.4016    &   \textbf{0.0834}   &  \textbf{0.0051}  \\
       Delta ailerons &    0.0538     &  0.4680 &\underline{0.0431}   &  1.3946   &  0.0417  &  3.5478   &   \underline{0.0453}    &  <\textbf{0.0001}  \\
       PD motor &   0.2318     &  0.4639     & 0.2241  &4.7680  & 0.2251 &  3.9016     &  \textbf{0.2210}    &  \textbf{0.0037}   \\
       PD total &   0.2178     &  0.4678     & \underline{0.2137}  &4.9278  & 0.2124 &  3.7854     &  \underline{0.2136}    &  \textbf{0.0023}   \\
       Puma &   0.1860     &  0.5070     & 0.1832  &2.1846  & 0.1830&  4.2161     &  \textbf{0.1808}    &  \textbf{0.0012}   \\
       Delta ele &   0.1223     &  0.5313     & \underline{0.1156}  &1.6206  & \underline{0.1155}&  4.0240     &  0.1174    &  <\textbf{0.0001}   \\
       Abalone &   0.0938     &  0.3398     & \underline{0.0808}  &1.2549  & \underline{0.0848}&  2.6676     &  \underline{0.0828}    &  \textbf{0.0017} \\
        Wine &   0.1360     &  0.3516     & \underline{0.1264}  &1.7098  & 0.1266  &  2.7126     & \underline{0.1250}    &  \textbf{0.0031} \\
        California house &   0.1801     &  1.1482     & 0.1450  &7.2625  & 0.1505  &  12.0832     & \underline{0.1420}    &  \textbf{0.0078} \\
        \bottomrule
\end{tabular}
\end{table*}

\begin{table*}[!htb]
 \centering
\caption{Performance comparison (mean-mean testing RMSE; time-training time)}
\begin{tabular}{ccccccc}
        \toprule
       Datasets & \multicolumn{2}{c}{EM-ELM} (200 nodes) & \multicolumn{2}{c}{ELM} (200 nodes)  &\multicolumn{2}{c}{the proposed method} (1 nodes)\\
                   \cmidrule{2-3}  \cmidrule{4-5}  \cmidrule{6-7}
      & Mean & time(s) & Mean & time(s) & Mean & time(s)\\
        \midrule
        House 8L &   \textbf{0.0663}   &   7.0388      & 0.0718   &0.8369   & 0.0819   &\textbf{0.0020}  \\
        Auto MPG  &  \underline{0.0968} & 0.0075  &  \underline{0.0976}   & 0.0156 & \underline{0.0996}  & <\textbf{0.0001}   \\
        Machine CPU  &  0.0521 & 0.1385  & 0.0513   &0.0069 & \textbf{0.0489} & <\textbf{0.0001}   \\
        Fried  &  \underline{0.0618} & 18.0290  & \underline{0.0619}   &1.3135 & 0.0834 & \textbf{0.0051}   \\
        Delta ailerons  &  \underline{0.0421} & 0.1342  & \underline{0.0431}   &0.0616 & \underline{0.0453} & <\textbf{0.0001}   \\
        PD motor  &  \underline{0.2196} & 0.7394  & \underline{0.2190}   &0.2730 & \underline{0.2203} & \textbf{0.0037}   \\
        PD total &  \underline{0.2094} & 0.5944  & \underline{0.2076}   &0.2838 & 0.2136 & \textbf{0.0023}   \\
        Puma &  \textbf{0.1478} & 4.8392  & 0.1602   &0.3728 & 0.1808 & \textbf{0.0012}   \\
        Abalone &  \underline{0.0817} & 0.1638  & \underline{0.0824}   &0.0761 & \underline{0.0828} & \textbf{0.0017}   \\
        Wine &  \underline{0.1216} & 0.3806  & \underline{0.1229}   &0.1950 & \underline{0.1250} & \textbf{0.0031}   \\
        California house &  \textbf{0.1302} & 3.5574  & 0.1354   &0.9753 & 0.1420 & \textbf{0.0078}   \\
        \bottomrule
\end{tabular}
\end{table*}

\begin{table}[!htb]
 \centering
\caption{Performance comparison (mean-mean testing RMSE; time-training time)}
\begin{tabular}{ccccc}
        \toprule
       Datasets & \multicolumn{2}{c}{ELM} (1 nodes) & \multicolumn{2}{c}{the proposed method} (1 nodes)  \\
                   \cmidrule{2-3}  \cmidrule{4-5}
      & Mean & time(s) & Mean & time(s) \\
        \midrule
        House 8L & 0.1083   &\textbf{0.0009}   &    \textbf{0.0819}       &   0.0020     \\
        Auto MPG  & 0.2126  & < \textbf{0.0001}   &  \textbf{0.0996}
        &  <\underline{0.0001}   \\
        Machine CPU      &0.1331 & <\underline{0.0001}     &   \textbf{0.0489}     &   <\underline{0.0001}   \\
     Fried   & 0.2207 & \textbf{0.0031}    &   \textbf{0.0834}   &  0.0051  \\
       Delta ailerons &  0.0864  &  <\underline{0.0001}   &    \textbf{0.0453}    &  <\underline{0.0001}  \\
       PD motor  & 0.2620 &  \textbf{0.0020}     &  \textbf{0.2210}    &  0.0037   \\
       PD total   & 0.2548 &  \textbf{0.0007}     &   \textbf{0.2136}    &  0.0023   \\
       Puma & 0.2856&  \underline{0.0012}     &  \textbf{0.1808}    &   \textbf{0.0012}    \\
       Delta ele  & 0.1454&  <\underline{0.0001}      &  \textbf{0.1174}    &  <\underline{0.0001}    \\
       Abalone   & 0.1363&  \textbf{0.0007}     &  \textbf{0.0828}    &  0.0017 \\
        Wine   & 0.1750  &  \textbf{0.0006}     &  \textbf{0.1250}    &  0.0031 \\
        California house  & 0.2496  &  \textbf{0.0027}     & \textbf{0.1420}    &  0.0078 \\
        \bottomrule
\end{tabular}
\end{table}

\begin{table*}[!htb]
 \centering
\caption{Performance comparison (mean-mean testing RMSE; time-training time)}
\begin{tabular}{ccccccc}
        \toprule
       Datasets & \multicolumn{2}{c}{Eplison-SVR}  & \multicolumn{2}{c}{BP}   &\multicolumn{2}{c}{the proposed method} (1 nodes)\\
                   \cmidrule{2-3}  \cmidrule{4-5}  \cmidrule{6-7}
      & Mean & time(s) & Mean & time(s) & Mean & time(s)\\
        \midrule
        House 8L &   \underline{0.0799}   &   53.6531      & \underline{0.0790}   &27.8462   & 0.0819   &\textbf{0.0020}  \\
        Auto MPG  &  \underline{0.0985} & 0.0234  &  \underline{0.0953}   & 1.6034 & \underline{0.0996}  & <\textbf{0.0001}   \\
        Machine CPU  &  0.0727 & 0.0187  & 0.0843   &0.7129 & \textbf{0.0489} & <\textbf{0.0001}   \\
        Fried  &  0.0829 & 197.9534  & \textbf{0.0591}   &81.8774 & 0.0834 & \textbf{0.0051}   \\
        Delta ailerons  &  \textbf{0.0402} & 6.8718  & 0.0415   &12.6735 & 0.0453 & <\textbf{0.0001}   \\
        California house &  0.1529 & 35.2250  & \underline{0.1435}   &54.3081 & 0.1420 & \textbf{0.0078}   \\
        PD total &  0.2082 & 7.2540  & 0.2120   &12.6438& 0.2136 & \textbf{0.0023}  \\
        \bottomrule
\end{tabular}
\end{table*}

\begin{table*}[!htb]
 \centering
\caption{Performance comparison (mean-mean testing RMSE; time-training time)}
\begin{tabular}{cccccccc}
        \toprule
       Datasets & \multicolumn{2}{c}{SVM}  & \multicolumn{3}{c}{ELM}   &\multicolumn{2}{c}{the proposed method} (1 nodes)\\
                   \cmidrule{2-3}  \cmidrule{4-6}  \cmidrule{7-8}
      & Mean & time(s) & Mean & time(s) & \#node& Mean & time(s) \\
        \midrule
        Covtype.binary &  74.84\%     &   413.5275       & \underline{77.27\%}   &36.5947  &500 & \underline{76.55\%}   & \textbf{1.2043}  \\
        Mushrooms  &  86.90\%  & 38.6247   &  46.97\%   & 0.9126 &500 & \textbf{88.84\%}  & \textbf{0.0047}   \\
        Gisette  &  77.68\%  &  309.3968  & 88.69\%   &\textbf{6.4093} &500& \textbf{94.10\%} & 48.2027   \\
        Leukemia  &  82.58\%  & \textbf{2.3914}   & 76.47\%   &9.0340 &5000 & \textbf{85.29\%} & 20.9915   \\
        W3a  & \underline{97.18\%}   & 4.5552   & \underline{97.25\%}   & 0.9095 &500& \underline{98.17\%} & \textbf{0.1872}   \\
        Duke  &  86.36\%  & \textbf{0.0156}   & 79.27\%   & \textbf{7.8437} &5000& \textbf{92.67\%} & 20.0352   \\
        Connect-4  & \underline{66.01\%}   & 569.6221   & \underline{76.55\%}   &7.3757 &500& \underline{75.40\%} & \textbf{0.7597}   \\
        Mnist  & 70.85\% &  478.4707  & \textbf{91.60\%}   &8.1651 &500  &84.20\% & 8.8858   \\
        DNA  & \underline{93.70\%}   &  0.4680  & 84.94\%   &0.2122 &500& \underline{92.41\%} & \textbf{0.0187}   \\
        \bottomrule
\end{tabular}
\end{table*}

\begin{table*}[!htb]
 \centering
\caption{Performance comparison (mean-mean testing RMSE; time-training time)}
\begin{tabular}{cccccccc}
        \toprule
       Datasets & \multicolumn{2}{c}{SVM}  & \multicolumn{3}{c}{ELM}   &\multicolumn{2}{c}{the proposed method} (1 nodes)\\
                   \cmidrule{2-3}  \cmidrule{4-6}  \cmidrule{7-8}
      & Mean & time(s) & Mean & time(s) & \#node& Mean & time(s) \\
        \midrule
        A9a &  77.39\%     &   295.0603       & \underline{85.10\%}   &4.5871  &500 & \underline{85.57\%}   & \textbf{0.5714}  \\
        Colon  & 76.67\%   & 10.0156   &  80.67\%   & 11.6283  &5000 & \textbf{85.06\%}  & \textbf{0.9719}   \\
        USPS  &  \underline{94.65\%}  &  146.4942  & \underline{93.54\%}   &2.0639 &500& 88.86\% & \textbf{0.4898}   \\
        Sonar  & \textbf{86.29\%}  & 0.0172   & 80.86\%   &0.0686 &500 & 75.69\% & <\textbf{0.0001}   \\
        Hill Valley&  58.67\%  & 0.1295   & 64.31   & 0.1647 &500& \textbf{67.61\%} & \textbf{0.0047}  \\
        Protein  & 51.18\%   & 253.5796   & 67.09\%   & 5.0919 &500& \textbf{68.76\%} & \textbf{1.9953}   \\
        \bottomrule
\end{tabular}
\end{table*}

\subsection{Real-world regression problems}

The experimental results between proposed the proposed method and some other incremental ELMs (B-ELM, I-ELM, and EI-ELM) are given in Table IV-Table V. In these tables, the close results obtained by different algorithms are underlined and the apparent better results are shown in boldface. All the incremental ELMs (I-ELM, B-ELM, EI-ELM) increase the hidden nodes one by one till nodes-numbers equal to 200, while for fixed ELMs (ELM, EM-ELM), 200-hidden-nodes are used. It can be seen that the proposed method can always achieve similar performance as other ELMs with much higher learning speed. In Table IV, for Machine CPU problem, the the proposed method runs 1900 times, 3400 times, 12000 times faster than the I-ELM, B-ELM and EI-ELM, respectively. For Abalone problem, the proposed method runs 200 times, 700 times, 1600 times faster than I-ELM, B-ELM and EI-ELM, respectively. In Table V, for Wine problem, the the proposed method runs 120 times and 60 times faster than EM-ELM and ELM, respectively. and the testing RMSE of EI-ELM is 2 times larger than the testing RMSE of B-ELM. The B-ELM runs 1.5 times faster than the I-ELM and the testing RMSE for the obained I-ELM is 5 times larger than the testing RMSE for B-ELM.

If only 1-hidden-node being used, those ELM methods such as I-ELM, ELM, EM-ELM and B-ELM can be considered as the same learning method (ELM[13]). Thus in Table VI, we carried out performance comparisons between the proposed method and 1-hidden-node ELM. As observed from Table VI, the average testing RMSE obtained by the proposed method are much better than the ELM. For California house and Delta ailerons problem, the testing RMSE obtained by ELM runs 2 times larger than that of the proposed method. In real applications, SLFNs with only 1 hidden nodes is extremely small network structure, meaning that after trained this small size network may response to new external unknown stimuli much faster and much more accurate than other ELM algorithms in real deployment.

\subsection{Real-world classification problems}

In order to indicate the advantage of the B-ELM on classification performance, the testing accuracy between the proposed the proposed method and other algorithms has also been conducted. Table VIII and IX display the performance comparison of SVM, ELM and the proposed method. In these tables, the close results obtained by different algorithms are underlined and the apparent better results are shown in boldface. As seen from those simulation results given in these tables, the proposed method can always achieve comparable performance as SVM and ELM with much faster learning speed. Take Covtype.binary (large number of training samples with medium input dimensions) and Gisette (medium number of training samples with high input dimensions).

1) For Covtype.binary data set, the proposed method runs 1403 times and 35 times faster than ELM and SVM, respectively.

2) For Gisette data set, the proposed method runs 341 times and 1.7 times faster than ELM and SVM, respectively.

Huang \emph{et al}.\cite{1Huang+Chen+Siew2003}\cite{Feng-Huang2009}\cite{4Huang2005}\cite{Huang2012} have systematically investigated the performance of ELM, SVM/SVR and BP for most data sets tested in this work. It is found that ELM obtain similar generalization performance as SVM/SVR but in much simpler and faster way. Similar to those above works, our testing results(cf. Table VII-IX) shows that the proposed the proposed method always provide comparable performance as SVM/SVR and BP with much faster learning speed.

On the other hand, the proposed method requires none human intervention than SVM, BP and other ELM methods. Different from SVM which is sensitive to the combinations of parameters ($C,\gamma$), or from other ELM methods in which parameter $C$ needs to be specified by users, the proposed method have none specified parameter and is ease of use in the respective implementations.

\section{Conclusion}

Unlike other SLFN learning methods, in our new approach, one may simply calculate the hidden node parameter once and the output weight is not need at all. And it has been rigorously proved that the proposed method can greatly enhance the learning effectiveness, reduce the computation cost, and eventually further increase the learning speed. The simulation results on sigmoid type of hidden nodes show that compared to other learning methods including SVM/SVR, BP and ELMs, the new approach can significantly reduce the NN training time several to thousands of times and can applied in regression and classification problems. Thus this method can be used efficiently in many applications.

However, we find an interesting phenomenon which we are not able to prove in this method, which should be worth pointing out. Experimental results show that this proposed learning method with one hidden node can achieve better generalization performance than the same method with $L,L>1$ hidden nodes. This phenomenon of this proposed method bring about many advantages, but if researchers can find the nature of this phenomenon, it can have far reaching consequences on the generalization ability of neural network.

\bibliography{reference}

\end{document}